\documentclass{l4dc2024}

\title[Safe OCO with Multi-Point Feedback]{Safe Online Convex Optimization with Multi-Point Feedback}
\usepackage{times}

\usepackage{amsmath}
\usepackage{color}
\usepackage{dsfont}
\allowdisplaybreaks
\DeclareMathOperator*{\argmin}{arg\,min}

\usepackage{cleveref}

\usepackage{algpseudocode}
\usepackage[algo2e,ruled]{algorithm2e}
\usepackage{tablefootnote}

\usepackage{mathtools}

\usepackage{lipsum}

\newtheorem{assumption}{Assumption}

\author{%
    \Name{Spencer Hutchinson} \Email{shutchinson@ucsb.edu}\\
    \addr University of California, Santa Barbara
    \AND
    \Name{Mahnoosh Alizadeh} \Email{alizadeh@ucsb.edu}\\
    \addr University of California, Santa Barbara%
    }

\begin{document}

\maketitle

\newcommand{\Fc}{\mathcal{F}}
\newcommand{\Fcvxc}{\mathcal{F}_{\mathrm{cvx}}}
\newcommand{\Xc}{\mathcal{X}}
\newcommand{\Fb}{\mathbb{F}}
\newcommand{\Fclin}{\mathcal{F}_{\mathrm{lin}}}
\newcommand{\Rb}{\mathbb{R}}
\newcommand{\Eb}{\mathbb{E}}
\newcommand{\Ac}{\mathcal{A}}
\newcommand{\reg}[3]{R_{#2}(#3)}
\newcommand{\Wc}{\mathcal{W}}
\newcommand{\Dc}{\mathcal{D}}
\newcommand{\Nb}{\mathbb{N}}
\newcommand{\Gc}{\mathcal{G}}
\newcommand{\Gb}{\mathbb{G}}
\newcommand{\Rtil}{\tilde{R}}
\newcommand{\lmin}{\mathrm{\lambda}_{\min}}
\newcommand{\lmax}{\mathrm{\lambda}_{\max}}
\newcommand{\smax}{\mathrm{\sigma}_{\max}}
\newcommand{\ttil}{\tilde{\tau}}
\newcommand{\xtil}{\tilde{x}}
\newcommand{\eb}{\mathbf{e}}
\newcommand{\Econf}{\mathcal{E}_{\mathrm{conf}}}
\newcommand{\Epe}{\mathcal{E}_{\mathrm{pe}}}
\newcommand{\rtil}{\tilde{r}}
\newcommand{\Bb}{\mathbb{B}}
\newcommand{\Sb}{\mathbb{S}}
\newcommand{\YS}{\mathcal{Y}^S}
\newcommand{\Pb}{\mathbb{P}}
\newcommand{\Yc}{\mathcal{Y}}
\newcommand{\Yctil}{\tilde{\mathcal{Y}}}
\newcommand{\Cc}{\mathcal{C}}
\newcommand{\Uc}{\mathcal{U}}
\newcommand{\thtil}{\tilde{\theta}}
\newcommand{\Oc}{\mathcal{O}}
\newcommand{\Kc}{\mathcal{K}}
\newcommand{\Atil}{\tilde{A}}
\newcommand{\atil}{\tilde{a}}
\newcommand{\util}{\tilde{u}}
\newcommand{\Bc}{\mathcal{B}}
\newcommand{\tr}{\mathrm{tr}}
\newcommand{\li}{\lambda_i}
\newcommand{\tamin}{\tau_{\mathrm{min}}}
\newcommand{\Octil}{\tilde{\mathcal{O}}}
\newcommand{\Hc}{\mathcal{H}}
\newcommand{\rd}[1]{\textcolor{red}{#1}}
\newcommand{\bl}[1]{\textcolor{blue}{#1}}
\newcommand{\ytil}{\tilde{y}}
\newcommand{\ti}{\mathrm{Term\ I}}
\newcommand{\tii}{\mathrm{Term\ II}}
\newcommand{\tia}{\mathrm{Term\ I.A}}
\newcommand{\tib}{\mathrm{Term\ I.B}}

\begin{abstract}%
    Motivated by the stringent safety requirements that are often present in real-world applications, we study a safe online convex optimization setting where the player needs to simultaneously achieve sublinear regret and zero constraint violation while only using zero-order information.
    In particular, we consider a multi-point feedback setting, where the player chooses $d + 1$ points in each round (where $d$ is the problem dimension) and then receives the value of the constraint function and cost function at each of these points.
    To address this problem, we propose an algorithm that leverages forward-difference gradient estimation as well as optimistic and pessimistic action sets to achieve $\Oc(d \sqrt{T})$ regret and zero constraint violation under the assumption that the constraint function is smooth and strongly convex.
    We then perform a numerical study to investigate the impacts of the unknown constraint and zero-order feedback on empirical performance.\footnote{This paper has been published in the proceedings of the Learning for Dynamics \& Control Conference (L4DC) 2024, available at \url{https://proceedings.mlr.press/v242/hutchinson24a}.}%
\end{abstract}


\section{Introduction}

The online convex optimization (OCO) problem, formalized by \cite{zinkevich2003online}, is a sequential decision-making framework where, in each round $t \in [T]$, a player chooses a vector action $x_t$ and subsequently observes the loss function $f_t$, with the goal of minimizing her cumulative loss $\sum_{t=1}^{T} f_t(x_t)$.
The OCO setting has received significant attention due to its practical effectiveness in various fields (e.g. online advertising (\cite{mcmahan2013ad}), network resource allocation (\cite{yu2019learning}) and power systems (\cite{lesage2019online})) and its role as a fundamental building block in modern learning and control approaches (e.g. online-to-batch (\cite{cutkosky2019anytime}) and online control (\cite{agarwal2019online,simchowitz2020improper})).
At the same time, there has been considerable recent interest in learning and control approaches that can ensure constraints are always satisfied, even when they are a priori unknown (e.g. \cite{sui2015safe,junges2016safety,usmanova2019safe}). 
This is motivated by safety-critical fields, such as clinical trials and power systems, where there is uncertainty about the constraints and constraint violation is not acceptable.
Accordingly, in this work, we consider an OCO setting with an unknown constraint that \emph{cannot be violated} while only giving the player \emph{partial feedback} on the constraint and cost functions.

In particular, we generalize the setting of OCO with multi-point feedback and known constraints from \cite{agarwal2010optimal} to the scenario where the constraint is unknown and the player only receives zero-order constraint information at the played actions.
Specifically, we consider an OCO setting where the player chooses multiple actions ($d+1$ actions to be precise) in each round, and then observes the cost function and constraint function values at each of these points.
Despite the limited information available, the player needs to ensure that all of the points that she chooses satisfy the constraints.
This is challenging because the player needs to effectively balance constraint satisfaction with regret minimization, while contending with errors in gradient estimation.
Note that this problem generalizes safe zero-order convex optimization as the cost functions do not change in that setting (i.e. $f_t = f$ for all $t$) and thus the player is freely able to query the cost function as desired.

To address the stated problem, we introduce the algorithm MP-ROGD, which combines the ideas from OCO under multi-point feedback (\cite{agarwal2010optimal}) with the idea of optimistic and pessimistic action sets from \cite{accsub}.
We rigorously show that when the player is given $d + 1$ points of feedback in each round (where $d$ is the problem dimension) and the constraint function is smooth and strongly-convex, MP-ROGD always satisfies the constraints and enjoys $\Oc(d\sqrt{T})$ regret.
Then, we perform a numerical study to assess the empirical performance of MP-ROGD against existing algorithms that either have access to zero-order cost information and complete constraint information (i.e. \cite{agarwal2010optimal}) or first-order cost and constraint information (i.e. \cite{accsub}).

\subsection{Related Work}

Constraints on the player's actions are a fundamental part of the OCO framework as even the initial formulation (\cite{zinkevich2003online}) assumes that the action set is bounded. 
However, this classical formulation assumes that these constraints are known, which may not be the case in some applications.
To address this gap, a large body of literature has emerged that studies OCO with time-varying constraints that are only revealed after the player commits to an action, e.g. \cite{mannor2009online,neely2017online, cao2018online, cao2018virtual, yi2020distributed, guo2022online}.
However, due to the limited information given to the player, these works aim for sublinear constraint violation rather than zero constraint violation.

In a different direction, several recent works have considered OCO with fixed constraints and zero constraint violation while providing the learner with limited information on the constraints (\cite{chaudhary2022safe,chang2023dynamic,accsub}).
In particular, \cite{chaudhary2022safe} gives an algorithm with $\Octil( T^{2/3})$ regret guarantees and high probability constraint satisfaction for an OCO setting with a linear constraint function and noisy feedback of the constraint function value at the chosen actions.
This method relies on an iid exploration phase within a small safe region to learn the constraint function everywhere, which cannot be readily applied to the nonlinear constraints considered in our setting.
The approach taken by \cite{chaudhary2022safe} is then extended to the distributed setting by \cite{chang2023dynamic}, where they additionally provide dynamic regret guarantees for both the cases of convex and non-convex cost functions.
\cite{accsub} take a different approach by assuming that the constraint function is smooth and strongly-convex, and give $\Oc(\sqrt{T})$ regret guarantees for the case when the player is given first-order feedback of the constraint function at the played actions.
In this work, we build on the approach taken in \cite{accsub} to address the more challenging setting where the player is only given multi-point zero-order feedback of the cost and constraint functions.

Another related area is ``projection-free'' OCO, which aims to develop OCO algorithms that do not require the computationally expensive projection operation.
Since projections require full knowledge of the constraint, there are some shared interests between projection-free OCO and OCO with unknown constraints.
One direction in projection-free OCO is focused on developing cheaper variants of the projection operation that can be used with standard algorithms, e.g. \cite{mhammedi2022efficient,levy2019projection}.
Another approach to projection-free OCO leverages the cheaper linear optimization oracle, which often involves variants of the Franke-Wolfe algorithm, e.g. \cite{garber2016linearly,hazan2020faster,kretzu2021revisiting}.
A third direction avoids projections by allowing some constraint violation, which shares some techniques with the literature on OCO with time-varying constraints, e.g. \cite{mahdavi2012trading,yu2020low,guo2022online}.
These approaches to projection-free OCO differ from the setting we consider in that they either allow constraint violation or assume access to different constraint oracles than we do, i.e. linear optimization oracle, membership oracle, or constraint function value and gradient at any point.

Our approach is also related to the literature on OCO with bandit feedback (first studied by \cite{flaxman2005online,kleinberg2004nearly}), where the learner is only given the cost function value at the played action (or sometimes several played actions) rather than the entire cost function ($f_t$) at each time step.
In fact, our setting can be considered a version of OCO with multi-point bandit feedback (\cite{agarwal2010optimal}) because we only give the player the cost function value at played actions.
As such, we borrow ideas from the multi-point OCO literature and the related zero-order optimization literature, e.g. \cite{agarwal2010optimal,duchi2015optimal}.
Furthermore, recent works that study zero-order optimization with unknown constraints are relevant (\cite{usmanova2020safe,guo2023safe}), although this setting is distinct from ours because it considers a fixed cost function, i.e. $f_t = f$ for all $t$.

\section{Preliminaries}

\subsection{Notation and Definitions}

We use $\Oc(\cdot)$ to refer to big-O notation.
Also, we denote the 2-norm by $\| \cdot \|$.
For a natural number $n$, we use $[n]$ for the set $\{ 1,2,...,n\}$.
For a matrix $M$, we use $M^\top$ to denote the transpose of $M$.
The unit vector in the $i$th coordinate direction is denoted by $e_i$.
A set $\Xc \subseteq \Rb^d$ is referred to as \emph{convex} if $(1 - \lambda) x + \lambda y \in \Xc$ for all $x,y \in \Xc$ and $\lambda \in [0,1]$.
For a convex set $\Xc$, a function $f: \Xc \rightarrow \Rb$ is referred to as \emph{convex} if $f((1 - \lambda)x + \lambda y) \leq (1 - \lambda)f(x) + \lambda f(y)$ for all $x,y \in \Xc$ and $\lambda \in [0,1]$.
Also for a closed convex set $\Xc \subseteq \Rb^d$ and a vector $x \in \Rb^d$, we denote the projection operation with $\Pi_\Xc (y) = \argmin_{x \in \Xc} \| x - y \|$.
A useful fact is that for a closed convex set $\Xc \subseteq \Rb^d$ and vectors $y \in \Rb^d$ and $x \in \Xc$, it holds that $\| y - x \| \geq \| \Pi_\Xc(y) - x \|$.
Lastly, we give the definitions for smooth and strongly convex functions which will be useful later.

\begin{definition}[Smooth function]
    Given a convex set $\Xc$, a differentiable convex function $h : \Xc \rightarrow \Rb$ is said to be $L$-\emph{smooth} if
    \begin{equation*}
        h(y) \leq h(x) + \nabla h(x)^\top (y-x) + \frac{L}{2} \| y - x \|^2
    \end{equation*}
    for all $x,y \in \Xc$.
\end{definition}

\begin{definition}[Strongly-convex function]
    Given a convex set $\Xc$, a differentiable convex function $h : \Xc \rightarrow \Rb$ is said to be $M$-\emph{strongly convex} if
    \begin{equation*}
        h(y) \geq h(x) + \nabla h(x)^\top (y-x) + \frac{M}{2} \| y - x \|^2
    \end{equation*}
    for all $x,y \in \Xc$.
\end{definition}

\subsection{Problem Setup}


\label{sec:prob}

We study an online convex optimization setting with $k = d+1$ points of zero-order feedback in each round and an unknown constraint.
This setting is defined by a \emph{horizon} $T \in \Nb$, a known closed convex \emph{action set} $\Xc \subseteq \Rb^d$, an unknown convex \emph{constraint function} $g : \Xc \rightarrow \Rb$, and a sequence of adversarially-chosen convex \emph{cost functions} $f_1, ..., f_T$ where $f_t : \Xc \rightarrow \Rb$ for every $t \in [T]$.
The setting can then be specified as an iterative game between a player and an adversary, where at each \emph{round} $t \in [T]$,
\begin{enumerate}
    \item player chooses actions $x_{t,1}, x_{t,2},...,x_{t,k}$ from $\Xc$,
    \item adversary chooses $f_t$ and player incurs the cost $\frac{1}{k} \sum_{i=1}^{k} f_t(x_{t,i})$,
    \item player observes $f_t (x_{t,1}),f_t (x_{t,2}),...,f_t (x_{t,k})$ and $g(x_{t,1}),g(x_{t,2}),...,g_t (x_{t,k})$.
\end{enumerate}
Despite the fact that $g$ is unknown, the player must ensure that $x_{t,1}, x_{t,2},...,x_{t,k}$ are in $\Gc := \{ x \in \Rb^d : g(x) \leq 0 \}$ for all $t \in [T]$.
We will refer to the \emph{feasible set} as $\Yc := \Xc \cap \Gc$.

In addition to maintaining constraint satisfaction, the player also aims to minimize her loss relative to the optimal action in hindsight.
That is, the player aims to minimize her \emph{regret}, which is defined as
\begin{equation*}
    R_T := \frac{1}{k} \sum_{t=1}^T \sum_{i=1}^{k} f_t(x_{t,i}) - \sum_{t=1}^T f_t(x_*),
\end{equation*}
where $x_* = \argmin_{x \in \Yc} \sum_{t=1}^T f_t(x)$. 
Note that this notion of regret is standard in OCO with multi-point feedback, i.e. \cite{agarwal2010optimal}.

\subsection{Assumptions}

Our approach to this problem uses several assumptions, which are given as follows.
First, we assume that the cost functions have bounded gradients (Assumption \ref{ass:cost_funcs}) and that the action set is bounded (Assumption \ref{ass:set_bound}), which are standard assumptions in the OCO setting, e.g. \cite{zinkevich2003online,hazan2016introduction}.

\begin{assumption}[Bounded gradients]
\label{ass:cost_funcs}
    For all $t \in [T]$, it holds that $f_t$ is differentiable and $\| \nabla f_t (x) \| \leq G$ for all $x \in \Xc$.
\end{assumption}

\begin{assumption}[Bounded action set]
\label{ass:set_bound}
    There exists a positive real $D$ such that $\| x - y \| \leq D$ for all $x,y \in \Xc$.
\end{assumption}

Next, we assume that the constraint function is smooth and strongly convex (Assumption \ref{ass:const}) and that there is a known point that is strictly feasible (Assumption \ref{ass:slaters}).
Assumption \ref{ass:const} is critical to our approach for ensuring low regret because it allows us to construct sets that tightly underestimate and overestimate the constraint set.
Assumption \ref{ass:slaters} ensures that there is a starting point that is known to satisfy the constraint, which is typically assumed in safe learning problems, e.g. \cite{usmanova2020safe,guo2023safe}.
In Assumption \ref{ass:slaters}, the player is also given both the radius of a ball that is within the constraint ($r$) and an upper bound on the function value at the starting point ($-\epsilon$).

\begin{assumption}[Smooth and strongly convex constraint]
\label{ass:const}
    The constraint function $g$ is differentiable, $L$-smooth and $M$-strongly convex, where $\kappa := L/M > 1$.\footnote{If $\kappa = 1$, then the constraint is exactly specified by the smoothness and strongly-convexity assumption, and the problem can be solved with standard OCO methods. Therefore, our assumption that $\kappa > 1$ is not restrictive.}
\end{assumption}

\begin{assumption}[Initial feasible point]
\label{ass:slaters}
    It holds that $\mathbf{0}$ is in $\mathcal{X}$ and $g(\mathbf{0}) \leq - \epsilon$ for some $\epsilon > 0$. Furthermore, there exists $r > 0$ such that $r \Bb \subseteq \Yc$.
\end{assumption}

Lastly, we assume that the cost functions are smooth, which ensures that the error in gradient estimation is small as in \cite{agarwal2010optimal, duchi2015optimal}.
Note that, unlike standard convex optimization, the OCO setting does not enjoy improved regret guarantees when the cost functions are smooth (see Table 3.1 in \cite{hazan2016introduction}).

\begin{assumption}[Smooth cost functions]
\label{ass:smooth}
    For all $t \in [T]$, it holds that $f_t$ is $L$-smooth.
\end{assumption}

\section{Proposed Algorithm}

\label{sec:alg}

\IncMargin{1em}
\begin{algorithm2e}[t]
\caption{Multi-point Restrained Online Gradient Descent (MP-ROGD)}
\label{alg:main_alg}
\DontPrintSemicolon
\LinesNumbered
\KwIn{$\Xc, G, L, M, r, \epsilon, \eta > 0, \delta \in (0,1), \alpha \in (0,1)$.}
Set $\xtil_{1} = \mathbf{0}$ and $x_{1} = \mathbf{0}$.\;
\For{$t=1$ \KwTo $T$}{
    Play $x_t, x_t + \delta e_1, x_t + \delta e_2, ..., x_t + \delta e_d$.\label{lne:acts} \;
    Set $\tilde{\nabla} f_t(x_t) = \frac{1}{\delta} \sum_{i=1}^{d} (f_t(x_t + \delta e_i) - f_t(x_t))e_i$ and $\tilde{\nabla} g(x_t) = \frac{1}{\delta} \sum_{i=1}^{d} (g(x_t + \delta e_i) - g(x_t))e_i$.\label{lne:grad_est}\;
    Update $\Yc_t^o$ and $\Yc_t^p$ with \eqref{eqn:optim} and \eqref{eqn:pess}.\;
    $\xtil_{t+1} = \Pi_{\Yc_t^o} (\xtil_t - \eta \tilde{\nabla} f_t(x_t) )$.\label{lne:opt_upd}\;
    $\gamma_t = \max\{\mu \in [0,1] : x_t + \mu (\xtil_{t+1} - x_t) \in \Yc_t^p \}$.\label{lne:safe_scal}\;
    $x_{t+1} = (1 - \alpha) (x_t + \gamma_t (\xtil_{t+1} - x_t))$.\label{lne:act_upd}\;
}
\end{algorithm2e}
\DecMargin{1em}

To address the stated problem, we propose the algorithm Multi-Point Restrained Online Gradient Descent (MP-ROGD) as stated in Algorithm \ref{alg:main_alg}.
MP-ROGD operates by using gradient estimators to approximate the gradients of the constraint and cost functions as described in Section \ref{sec:grad_est}, and then leveraging optimistic and pessimistic action sets to ensure small regret while maintaining constraint satisfaction as described in Section \ref{sec:opt_pess}.
We give guarantees that the algorithm is well-defined and that it never violates the constraints in Section \ref{sec:val_safe}.
The regret of MP-ROGD is studied in the following section (Section \ref{sec:reg}).
The proofs from this section are given in Appendix \ref{apx:alg}.

\subsection{Gradient Estimation}

\label{sec:grad_est}

Because the algorithm does not have access to gradients of the cost functions or the constraint function, it estimates the gradients with only zero-order information.
The algorithm does this by playing the current iterate $x_t$ as well as points perturbed away from the current iterate by $\delta$ in each coordinate direction (given in line \ref{lne:acts}).
It then estimates the gradient at the current iterate using forward difference (line \ref{lne:grad_est}).
We give some useful properties of this gradient estimator in the following proposition.
Note that an appropriate choice for $\delta$ will be specified later.

\begin{proposition}[Properties of gradient estimators]
    \label{prop:grad_est}
    Let Assumptions \ref{ass:cost_funcs}, \ref{ass:const} and \ref{ass:smooth} hold.
    Then, for every $t \in [T]$, it holds that
    \begin{equation*}
        \| \tilde{\nabla} f_t(x_t) - \nabla f_t(x_t) \| \leq \frac{1}{2} \sqrt{d} L \delta \quad \text{and} \quad \| \tilde{\nabla} g(x_t) - \nabla g(x_t) \| \leq \frac{1}{2} \sqrt{d} L \delta.
    \end{equation*}
    Furthermore, it holds that
    \begin{equation*}
        \| \tilde{\nabla} f_t(x_t) \| \leq d G.
    \end{equation*}
\end{proposition}

The key takeaways from Proposition \ref{prop:grad_est} are that the gradient estimation error shrinks as $\delta$ shrinks and that the norm of the gradient estimator can be bounded independently of $\delta$.
Since regret will grow as both gradient estimation error and the norm of the gradient estimator increases, Proposition~\ref{prop:grad_est} tells us that we can take $\delta$ to be small without sacrificing regret.
This is important because a large $\delta$ might otherwise jeopardize constraint satisfaction, and therefore taking $\delta$ to be sufficiently small (see the choice of $\delta$ in Theorem \ref{thm:main}) will allow for both low regret and constraint satisfaction.

\subsection{Optimistic and Pessimistic Action Sets}

\label{sec:opt_pess}

The proposed algorithm updates the iterate $x_t$ using a technique that leverages both an \emph{optimistic action set} (denoted by $\Yc_t^o$) and a \emph{pessimistic action set} (denoted by $\Yc_t^p$), which are known to contain the true feasible set and be contained by the true feasible set, respectively.
We refer to $\Yc_t^o$ (resp. $\Yc_t^p$) as the optimistic (resp. pessimistic) action set because it estimates the feasible set while taking the unknown information about the constraint to be as favorable (resp. unfavorable) as reasonably possible given what has been observed.\footnote{We borrow this terminology from the stochastic bandit literature (e.g. \cite{abbasi2011improved}) where ``optimism in the face of uncertainty'' is a popular design paradigm.}
The algorithm uses these sets in each round by updating an \emph{optimistic iterate} ($\tilde{x}_t$) with gradient descent on the optimistic set (line \ref{lne:opt_upd}) and then moving the \emph{played iterate} ($x_t$) towards the optimistic iterate while keeping it in the pessimistic set (line \ref{lne:act_upd}).
This ensures that the optimistic iterates incur low regret while simultaneously keeping the played iterates within the constraint set.
The specific construction of the optimistic and pessimistic action sets, which we discuss next, ensures that the played iterates stay near to the optimistic iterates and therefore that the played iterates incur low regret as well.
Note that the played iterates are scaled down by $(1 - \alpha)$ in line \ref{lne:act_upd} to ensure that the perturbed points ($x_t + \delta e_i$) do not violate the constraints.

The optimistic and pessimistic action sets are constructed by combining the smoothness and strong-convexity of the constraint function with the error bound on the gradient estimator in Proposition \ref{prop:grad_est}.
Specifically, the optimistic and pessimistic action sets are defined as
\begin{equation}
    \label{eqn:optim}
        \Yc_t^o := \left\{ x \in \Xc : g(x_t)  - \frac{1}{2} \sqrt{d} L \delta D + \tilde{\nabla} g(x_t)^\top (x - x_t) + \frac{M}{2} \left\| x - x_t \right\|^2 \leq 0 \right\},
\end{equation}
and,
\begin{equation}
    \label{eqn:pess}
    \Yc_t^p := \left\{ x \in \Xc : g(x_t)  + \frac{1}{2} \sqrt{d} L \delta D + \tilde{\nabla} g(x_t)^\top (x - x_t) + \frac{L}{2} \left\| x - x_t \right\|^2 \leq 0 \right\}
\end{equation}
respectively.
In the following proposition, we show that the optimistic and pessimistic sets do in fact overestimate and underestimate the constraint set, respectively.
\begin{proposition}
    \label{prop:opt_pess}
    Let Assumptions \ref{ass:set_bound} and \ref{ass:const} hold. Then, it follows that $\Yc_t^p \subseteq \Yc \subseteq \Yc_t^o$ for all $t$.
\end{proposition}

\subsection{Validity and Safety Gaurantee}

\label{sec:val_safe}

It is necessary to show that the algorithm is well-defined and that the constraint is satisfied at all rounds.
The main point of concern is whether the pessimistic set $\Yc_t^p$ is nonempty.
In the following proposition, we provide a range of values of $\delta$ for which the pessimistic set is guaranteed to be nonempty.

\begin{proposition}[Validity]
    \label{prop:viab}
    Let Assumptions \ref{ass:const} and \ref{ass:slaters} hold.
    If $\delta \leq \frac{2 \alpha \epsilon}{\sqrt{d} L D}$, then $x_t \in \Yc_t^p$ (and therefore $\Yc_t^p$ is nonempty) for all rounds $t \in [T]$.
\end{proposition}

Next, we show that all actions satisfy the constraint if $\delta$ is chosen appropriately.

\begin{proposition}[Safety guarantee]
    \label{prop:safe}
    Let Assumption \ref{ass:const} hold and assume that $x_t \in \Yc_t^p$ for all $t \in [T]$. If $\delta \leq \alpha r$, then all actions played by the algorithm, i.e. $x_t, x_t + \delta e_1, ..., x_t + \delta e_d$ for all $t$, are in the feasible set $\Yc$.
\end{proposition}

\section{Regret Analysis}

\label{sec:reg}

In the following theorem, we show that, with an appropriate choice of algorithm parameters ($\alpha, \delta, \eta$), our proposed algorithm MP-ROGD (Algorithm \ref{alg:main_alg}) enjoys $\Oc(d \sqrt{T})$ regret and ensures that the constraint is always satisfied.
A proof sketch of Theorem \ref{thm:main} and the supporting lemmas are given below and the complete proof is given in Appendix \ref{apx:reg}.

\begin{theorem}
    \label{thm:main}
    Let Assumptions \ref{ass:cost_funcs}, \ref{ass:set_bound}, \ref{ass:const} and \ref{ass:slaters} hold.
    If $\alpha = \min(0.5,\frac{d G}{D} (1 - \frac{1}{\kappa}) \eta )$ and 
    \begin{equation*}
        \delta = \min\left(\frac{1}{\left( \frac{1}{2} \sqrt{d} L D + G \right) T}, \frac{2 (\kappa - 1) \alpha \epsilon}{(\kappa + 1) \sqrt{d} L D}, \alpha r \right),
    \end{equation*}
    then all actions chosen by MP-ROGD (Algorithm \ref{alg:main_alg}) satisfy the constraint, and the regret  satisfies
    \begin{equation*}
        R_T \leq 2 d^2 G^2 \left( \kappa - \frac{3}{4} \right) \eta T  + \frac{ D^2}{2 \eta} + 1.
    \end{equation*}
    Furthermore, choosing $\eta = \frac{D}{2 \sqrt{\left( d/4 +  \kappa - 1 \right) d G^2 T}}$ ensures that 
    \begin{equation*}
        R_T \leq 2 D G \sqrt{d \left( \frac{1}{4} d +  \kappa - 1 \right) T} + 1.
    \end{equation*}
\end{theorem}

\paragraph{Proof sketch:}
First, we seperate the regret due to the iterate $x_t$ from the regret due to the perturbed iterates $x_t + \delta e_1, ..., x_t + \delta e_d$ as
\begin{align*}
    R_T & =  \frac{1}{k} \sum_{t=1}^{T} \sum_{i=1}^{k} f_t(x_{t,i}) - \sum_{t=1}^{T} f_t(x_*)\\
    & = \underbrace{\sum_{t=1}^{T} \left( f_t(x_t) - f_t(x_*) \right)}_{\ti} + \underbrace{\sum_{t=1}^{T} \left( \frac{1}{d + 1} \left( f_t(x_t) +  \sum_{i=1}^{d} f_t(x_t + \delta e_i) \right) - f_t(x_t) \right)}_{\tii},
\end{align*}
and note that $\tii \leq T G \delta$ given that the gradient of $f_t$ is assumed to be bounded by $G$ in Assumption~\ref{ass:cost_funcs}.
Then, we decompose $\ti$ as
\begin{equation}
    \label{eqn:ti}
    \begin{split}
    \ti = \sum_{t=1}^{T} \left( f_t(x_t) - f_t(x_*) \right) & \leq \sum_{t=1}^{T} \nabla f_t (x_t)^\top \left( x_t - x_* \right)\\
    & = \underbrace{\sum_{t=1}^{T} \nabla f_t (x_t)^\top \left( x_t - \xtil_t \right)}_{\tia} + \underbrace{\sum_{t=1}^{T} \nabla f_t (x_t)^\top \left( \xtil_t - x_* \right)}_{\tib},
    \end{split}
\end{equation}
where the inequality is due to convexity (using the idea from \cite{zinkevich2003online} of studying the linearized regret).
$\tia$ is the difference in (linearized) cost between the played iterate $x_t$ and the optimistic iterate $\xtil_t$, while $\tib$ can be interpreted as the linearized regret due to the optimistic iterate.
In Lemmas \ref{lem:gam} and \ref{lem:dist} in the following subsection, we show that the specific structure of the optimistic and pessimistic sets ensures that the distance between $x_t$ and $\xtil_t$ is small and therefore that $\tia$ is small.
Furthermore, the optimistic iterates are updated with gradient descent on the optimistic set, which is known to contain the true feasible set, so we can apply techniques from multi-point OCO (\cite{agarwal2010optimal}) to bound $\tib$.
This approach uses Lemma~\ref{lem:lin_reg}, which is given in the following subsection.

\subsection{Supporting Lemmas}

The proof of Theorem \ref{thm:main} relies on three key lemmas that are given in this section.
The first two lemmas (Lemmas \ref{lem:gam} and \ref{lem:dist}) establish a bound on the distance between the played iterates $x_t$ and the optimistic iterates $\xtil_t$, while the third lemma (Lemma \ref{lem:lin_reg}) establishes a bound on the linearized regret of the optimistic iterate.
In particular, Lemma \ref{lem:gam} (given in the following) shows that $\gamma_t$ is always larger than $1/\kappa$ when $\delta$ is chosen sufficiently small.

\begin{lemma}
    \label{lem:gam}
    Let Assumptions \ref{ass:const} and \ref{ass:slaters} hold.
    If $\delta \leq \frac{2 (\kappa - 1) \alpha \epsilon}{(\kappa + 1) \sqrt{d} L D}$, then $\gamma_t \geq 1/\kappa$ for all $t \in [T]$.
\end{lemma}

This result is then used in Lemma \ref{lem:dist} to show that (when $\delta$ is sufficiently small) the distance between the optimistic iterate $\xtil_t$ and played iterate $x_t$ is always bounded by a value proportional to $\eta$ and $\alpha$.
Since $\alpha$ has no other restrictions, we can choose $\alpha$ to be proportional to $\eta$ and therefore Lemma \ref{lem:dist} tells us that $\| x_t - \xtil_t \| \leq \Oc(\eta)$.
At the same time, $\eta$ needs to be chosen as $\Theta(\frac{1}{\sqrt{T}})$ to ensure optimal regret for gradient descent-based algorithms.
As it happens, Lemma \ref{lem:dist} implies that such a choice of $\eta$ also ensures that $\tib$ in \eqref{eqn:ti} is $\Oc(\sqrt{T})$, i.e. that $\sum_{t=1}^{T} \| x_t - \xtil_t \| \leq \Oc(\sqrt{T})$.


\begin{lemma}
    \label{lem:dist}
    Let Assumptions \ref{ass:cost_funcs}, \ref{ass:const} and \ref{ass:slaters} hold.
    Fix any $\rho > 0$.
    If $\eta \leq \frac{1/\kappa}{2 d G(1-1/\kappa)} \rho$, $\alpha \leq \frac{1}{2 D \kappa} \rho$ and $\delta \leq \frac{2 (\kappa - 1) \alpha \epsilon}{(\kappa + 1) \sqrt{d} L D}$, then it holds that $\| x_t - \xtil_t \| \leq \rho$ for all~$t$.
\end{lemma}

Lastly, we give Lemma \ref{lem:lin_reg}, which provides a bound on the (estimated) linearized regret of the optimistic iterates.
In particular, it is easy to see that by summing \eqref{eqn:lin_reg} over $t$, the righthand side telescopes, yielding the bound $\frac{1}{\eta} D^2 + \frac{1}{2} \eta d^2 G^2 T$.
Choosing $\eta = \Theta(1/\sqrt{T})$ ensures that this is $\Oc(\sqrt{T})$.
This can then be used to bound $\tib$ in \eqref{eqn:ti}, although there will be an additive $\frac{1}{2} \sqrt{d} L \delta R T$ (due to Proposition \ref{prop:grad_est}) because \eqref{eqn:lin_reg} is in terms of the estimated gradient $\tilde{\nabla} f_t$ rather than the true gradient $\nabla f_t$.
However, choosing $\delta \leq \frac{1}{T}$ ensures that this $\Oc(1)$.  

\begin{lemma}
    \label{lem:lin_reg}
    Let Assumptions \ref{ass:cost_funcs} and \ref{ass:const} hold.
    Then, for any $v \in \Yc$, it holds that
    \begin{equation}
        \label{eqn:lin_reg}
        \tilde{\nabla} f_t (x_t)^\top (\xtil_t - v) \leq \frac
        {1}{2 \eta} \left( \| \xtil_t - v \|^2 - \| \xtil_{t+1} - v \|^2 \right) + \frac{1}{2} \eta d^2 G^2,
    \end{equation}
    for all $t \in [T]$.
\end{lemma}

\section{Numerical Experiments}

\begin{figure}
    \centering
    \subfigure{
        \includegraphics[width=0.45\textwidth]{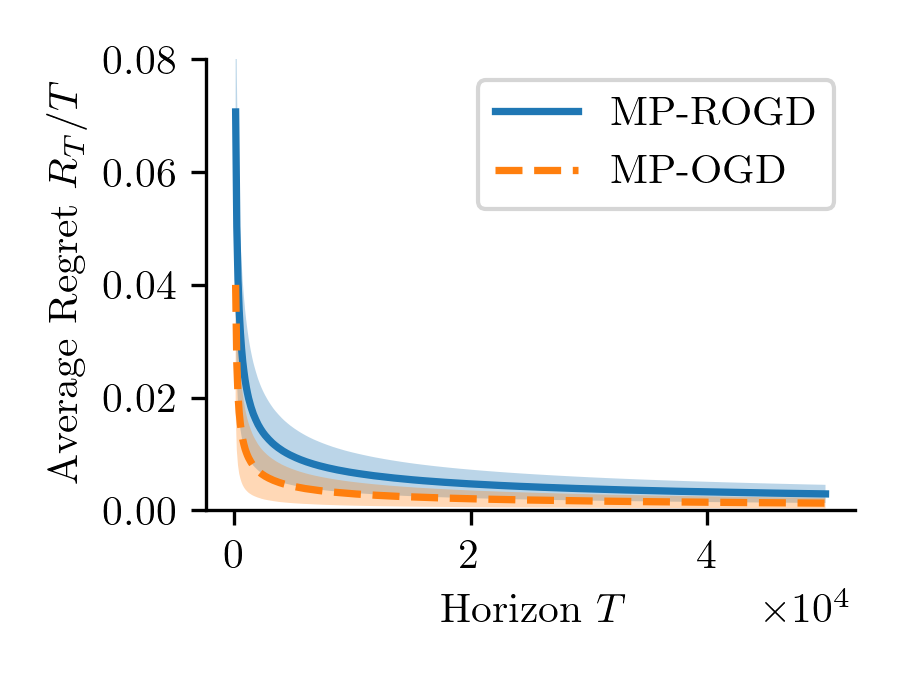}
        \label{fig:expers:a}
        }
        \hspace{0.05\textwidth}
    \subfigure{
        \includegraphics[width=0.45\textwidth]{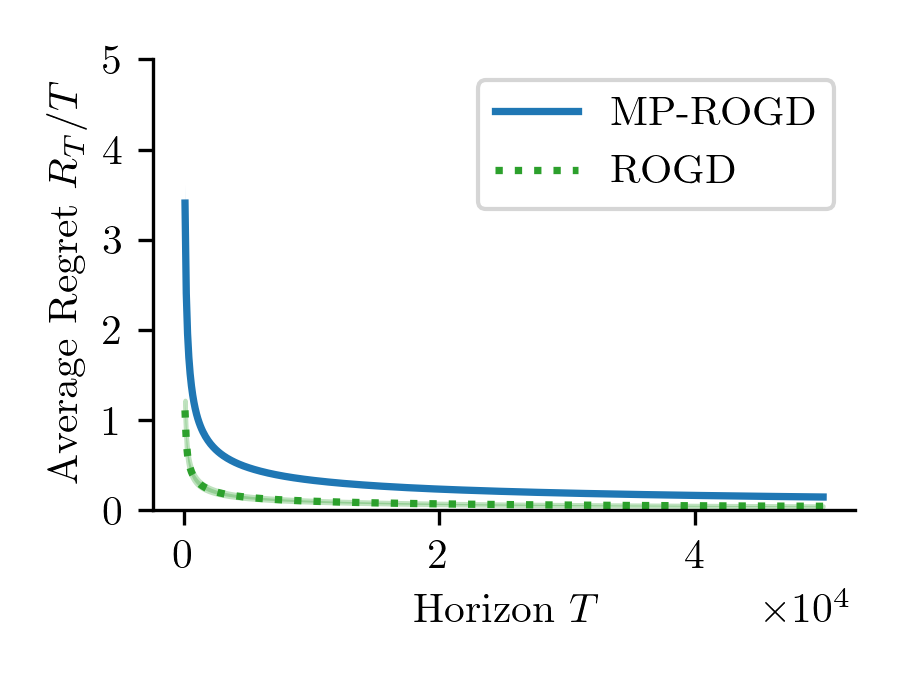}
        \label{fig:expers:b}
        }
    \caption{Average regret of MP-ROGD and benchmark algorithms in a setting with linear cost functions and a quadratic constraint function (a) and a setting with quadratic cost functions and quadratic constraint (b). The benchmark algorithms are MP-OGD (\cite{agarwal2010optimal}) with full constraint information and ROGD (\cite{accsub}) with first-order constraint feedback.}
    \label{fig:expers}
    \vspace{-0.2in}
\end{figure}
In order to assess the empirical performance of MP-ROGD, we compare MP-ROGD with two different benchmark algorithms in toy experimental settings.
In the first experimental setting, we study the impact of unknown constraints by running MP-ROGD alongside multi-point online gradient descent (\cite{agarwal2010optimal}) which uses full constraint information, and in the second setting, we study the impact of zero-order feedback by running MP-ROGD alongside ROGD (\cite{accsub}) which uses first-order feedback.

\subsection{Impact of unknown contraints}
\label{sec:unk_cont}

To study the impact of unknown constraints on empirical performance, we compare MP-ROGD to online gradient descent with $d + 1$ points of feedback from \cite{agarwal2010optimal} (abbreviated MP-OGD) which uses the full constraint information.
We run these algorithms in a toy setting with cost functions $f_t(x) = \theta_t^\top x$ with $\theta_t \sim \Uc[0,1]^d$ and constraint function of the form $g(x) = a \| x - b \|^2 + c$ where the problem dimension is $d = 2$.
We consider $10$ randomly sampled settings of this form, where $a \sim \Uc[1,10]$, $b \sim \Uc(0.2 \Sb)$, $c = - \xi^2 a$ and $\xi \sim \Uc[0.3,0.8]$.
We sample the problem parameters in this manner because it ensures that we can take $\Xc = \Bb$ such that $\Gc \subseteq \Xc$ which allows for easy computation, and $r = 0.1$ in the sense of Assumption \ref{ass:slaters}.
Furthermore, we take $G = \sqrt{2}$ (Assumption \ref{ass:cost_funcs}), $R = 2$ (Assumption \ref{ass:set_bound}), $\epsilon = -c, r = 0.1$ (Assumption \ref{ass:slaters}), and $L = 20$, $M = 2$ (Assumption \ref{ass:const}).
Note that the constraint function is $2a$-smooth and $2a$-strongly convex, but the player does not know this, so we only provide the player with the information that $L = 20$, $M = 2$ which can be deduced from the sampling distribution for $a$.
For MP-ROGD, we choose $\eta = \frac{R}{d G \sqrt{T}}$, $\alpha = d GM(1 - 1/\kappa) \eta/R$ and $\delta = \mathrm{min}(1/T,(\kappa - 1) \alpha \epsilon/((\kappa + 1) \sqrt{d} L R), \alpha r)$ which satisfies the conditions in Theorem~\ref{thm:main} for $\Oc(d \sqrt{T})$ regret and no constraint violation.
For MP-OGD, which is specified by the update $x_{t+1} = \Pi_{(1 - \alpha)\Yc} (x_t - \eta \tilde{\nabla} f_t (x_t))$, we choose $\eta = \frac{R}{d G \sqrt{T}}$, $\delta = 1/T$ $\alpha = \delta/ \bar{r}$ where $\bar{r} = \xi - 0.2$ (the largest ball radius that is within the constraint).

The results of these experiments are shown in Figure \ref{fig:expers:a}.
These results are generated by running both algorithms in each randomly sampled setting for every $T \in \{1\times 10^2,2\times 10^2,...,5\times 10^4\}$ and calculating the average regret $R_T/T$ for each.
The average and standard deviation of $R_T/T$ across settings is shown in Figure \ref{fig:expers:a}.
From these results, we can see that there is a significant performance gap between MP-ROGD with only zero-order constraint feedback, and MP-OGD with full constraint information.
Notably, this differs from the case of first-order feedback, for which \cite{accsub} observed little performance difference between ROGD with first-order feedback and online gradient descent with full constraint information.
This suggests that the ``price of safety'' increases as less constraint information is given to the player.

\subsection{Impact of zero-order feedback}

To study the impact of multi-point feedback on the empirical performance of safe OCO algorithms, we compare MP-ROGD with ROGD from \cite{accsub} which uses first-order constraint feedback.
We run these algorithms in a toy setting with cost functions $f_t (x) = (x - b_t)^\top A_t (x - b_t)$ where $A_t$ and $b_t$ are randomly sampled in each round, and constraint function $g (x) = x^\top \tilde{A} x + \tilde{c}$.
We generate $A_t$ in each round by sampling $A_{t,\mathrm{raw}} \sim \Uc[0,1]^{d \times d}$, taking the symmetric part $A_{t,\mathrm{sym}} = 0.5(A_{t,\mathrm{raw}} + A_{t,\mathrm{raw}}^\top)$, normalizing its spectrum $A_{t,\mathrm{norm}} = (A_{t,\mathrm{sym}} - 0.5 I)/(d - 0.5)$ and finally by shifting and scaling $A_t = 5(A_{t,\mathrm{norm}} + I)$ to ensure the spectrum is within $[1,10]$.
Also, we sample $b_t \sim \Uc[1,2]^d$ in each round which will ensure that the constraint is tight on the optimal action.
We consider 10 randomly sampled settings with $\tilde{A} = \mathrm{diag}(\tilde{a})$ and $\tilde{c} = \min_i(\tilde{a}_i)$, where $\tilde{a} \sim \Uc[1,10]^d$.
Similar to Section \ref{sec:unk_cont}, this ensures that $\Gc \subseteq \Xc$ when $\Xc = \Bb$.
Furthermore, we choose the problem parameters $G = 60$ (Assumption \ref{ass:cost_funcs}), $R = 2$ (Assumption \ref{ass:set_bound}), $\epsilon = 1, r = 1/\sqrt{10}$ (Assumption \ref{ass:slaters}), $L = 20, M = 2$ (Assumption \ref{ass:const}, Assumption \ref{ass:smooth}).
We choose algorithm parameters of MP-ROGD as $\eta = \frac{R}{d G \sqrt{T}}$, $\alpha = d GM(1 - 1/\kappa) \eta/R$ and $\delta = \mathrm{min}(1/T,(\kappa - 1) \alpha \epsilon/((\kappa + 1) \sqrt{d} L R), \alpha r)$ (same as in Section \ref{sec:unk_cont}).
Also, we run ROGD with $\eta = \frac{R}{G \sqrt{T}}$ as suggested in \cite{accsub}.

The results of these experiments are shown in Figure \ref{fig:expers:b}, which is computed the same as for the results in Section \ref{sec:unk_cont}.
These results show that ROGD outperforms MP-ROGD, suggesting that there is a cost to only having zero-order feedback versus first-order feedback.  



\section{Conclusion}

In this work, we study a safe OCO problem where the player chooses $d + 1$ actions in each round and observes the cost and constraint values at each of these points. 
To address this problem, we present the algorithm MP-ROGD, which enjoys $\Oc(d \sqrt{T})$ regret and never violates the constraints.
One interesting direction for future work is investigating whether it is possible to do safe OCO under nonlinear constraints with less constraint information (e.g. one or two-point feedback), although this might require weaker notions of constraint satisfaction (e.g. in expectation).
Another interesting direction for future work is investigating whether our proposed algorithmic approach can be applied to related learning problems such as distributed online optimization or online control.

\acks{This work was supported by NSF grant \#1847096.}

\bibliography{references}

\newpage

\appendix

\section{Proofs from Section \ref{sec:alg}}
\label{apx:alg}

In the following subsections, the proofs from Section \ref{sec:alg} are restated and then proved.
Specifically, the proofs of Propositions \ref{prop:grad_est}, \ref{prop:opt_pess}, \ref{prop:viab}, \ref{prop:safe} are given in Appendices \ref{apx:grad}, \ref{apx:opt_pess}, \ref{apx:viab}, \ref{apx:safe} respectively.

\subsection{Proof of Proposition \ref{prop:grad_est}}
\label{apx:grad}

\begin{proposition}[Duplicate of Proposition \ref{prop:grad_est}]
    \label{prop:grad_est2}
    Let Assumptions \ref{ass:cost_funcs}, \ref{ass:const} and \ref{ass:smooth} hold.
    Then, for every $t \in [T]$, it holds that
    \begin{equation*}
        \| \tilde{\nabla} f_t(x_t) - \nabla f_t(x_t) \| \leq \frac{1}{2} \sqrt{d} L \delta \quad \text{and} \quad \| \tilde{\nabla} g(x_t) - \nabla g(x_t) \| \leq \frac{1}{2} \sqrt{d} L \delta.
    \end{equation*}
    Furthermore, it holds that
    \begin{equation*}
        \| \tilde{\nabla} f_t(x_t) \| \leq d G.
    \end{equation*}
\end{proposition}
\begin{proof}
    This proof is fairly standard (e.g. \cite{agarwal2010optimal}), but we give it for completeness.
    First, we have from smoothness that
    \begin{equation}
        \label{eqn:up_bound}
        \begin{split}
        & f_t(x_t + \delta e_i) \leq f(x_t) + \delta \nabla f_t(x_t)^\top e_i + \frac{L \delta^2}{2}\\
        & \Longleftrightarrow \quad \frac{1}{\delta} (f_t(x_t + \delta e_i) - f(x_t)) - \nabla f_t(x_t)^\top e_i \leq \frac{L \delta}{2}.
        \end{split}
    \end{equation}
    Then from convexity,
    \begin{equation}
        \begin{split}
        \label{eqn:down_bound}
        & f_t(x_t + \delta e_i) \geq f(x_t) + \delta \nabla f_t(x_t)^\top e_i\\
        & \Longleftrightarrow \quad \frac{1}{\delta} (f_t(x_t + \delta e_i) - f(x_t)) - \nabla f_t(x_t)^\top e_i \geq 0.
        \end{split}
    \end{equation}
    Combining \eqref{eqn:up_bound} and \eqref{eqn:down_bound} yields $|\frac{1}{\delta} (f_t(x_t + \delta e_i) - f(x_t)) - \nabla f_t(x_t)^\top e_i| \leq L \delta /2$.
    It follows that
    \begin{equation*}
        \| \tilde{\nabla} f_t(x_t) - \nabla f_t(x_t) \| = \sqrt{\sum_{i=1}^{d} \left| \frac{1}{\delta} (f_t(x_t + \delta e_i) - f(x_t)) - \nabla f_t(x_t)^\top e_i \right|^2} \leq \frac{1}{2} \sqrt{d} L \delta,
    \end{equation*}
    which gives the first inequality in the statement of the proposition.
    Since the constraint function $g$ is also assumed to be smooth and convex, the second inequality follows immediately.
    Lastly, the third inequality comes from the assumption that $f_t$ has bounded gradients as
    \begin{align*}
        \| \tilde{\nabla} f_t(x_t) \| & = \left\| \frac{1}{\delta} \sum_{i=1}^{d} (f_t(x_t + \delta e_i) - f_t(x_t)) e_i \right\|\\
        &  \leq \frac{1}{\delta} \sum_{i=1}^{d} \| (f_t(x_t + \delta e_i) - f_t(x_t)) e_i \| \leq \frac{1}{\delta} d G \delta = d G,
    \end{align*}
    completing the proof.
\end{proof}

\subsection{Proof of Proposition \ref{prop:opt_pess}}
\label{apx:opt_pess}

\begin{proposition}[Duplicate of Proposition \ref{prop:opt_pess}]
    \label{prop:opt_pess2}
    Let Assumptions \ref{ass:set_bound} and \ref{ass:const} hold. Then, it follows that $\Yc_t^p \subseteq \Yc \subseteq \Yc_t^o$ for all $t$.
\end{proposition}
\begin{proof}
    Firstly, it holds for all $y \in \Yc_t^p$ that
    \begin{align*}
        g(y) & \leq g(x_t) +  \nabla g(x_t)^\top (y - x_t) + \frac{L}{2} \left\| y - x_t \right\|^2\\
        & \leq g(x_t) + \| \nabla g(x_t) - \tilde{\nabla} g(x_t)\| \| y - x_t \| + \tilde{\nabla} g(x_t)^\top (y - x_t) + \frac{L}{2} \left\| y - x_t\right\|^2\\
        & \leq g(x_t) + \frac{1}{2} \sqrt{d} L \delta D + \tilde{\nabla} g(x_t)^\top (y - x_t) + \frac{L}{2} \left\| y - x_t\right\|^2 \leq 0
    \end{align*}
    where the first inequality comes from the smoothness assumption on $g$, the second inequality is Cauchy-Schwarz and the third inequality uses the error bound on the gradient estimator in Proposition~\ref{prop:grad_est} and Assumption \ref{ass:set_bound}.
    It follows that $\Yc_t^p \subseteq \Yc$.
    Then, using the strong-convexity assumption, it holds for all $y \in \Yc$ that
    \begin{align*}
        0 \geq g(y) & \geq g(x_t) +  \nabla g(x_t)^\top (y - x_t) + \frac{M}{2} \left\| y - x_t \right\|^2\\
        & \geq g(x_t) - \| \nabla g(x_t) - \tilde{\nabla} g(x_t)\| \| y - x_t \| + \tilde{\nabla} g(x_t)^\top (y - x_t) + \frac{M}{2} \left\| y - x_t\right\|^2\\
        & \geq g(x_t) - \frac{1}{2} \sqrt{d} L \delta D + \tilde{\nabla} g(x_t)^\top (y - x_t) + \frac{M}{2} \left\| y - x_t\right\|^2,
    \end{align*}
    and therefore $\Yc \subseteq \Yc_t^o$.
\end{proof}

\subsection{Proof of Proposition \ref{prop:viab}}
\label{apx:viab}

\begin{proposition}[Duplicate of Proposition \ref{prop:viab}]
    \label{prop:viab2}
    Let Assumptions \ref{ass:const} and \ref{ass:slaters} hold.
    If $\delta \leq \frac{2 \alpha \epsilon}{\sqrt{d} L D}$, then $x_t \in \Yc_t^p$ (and therefore $\Yc_t^p$ is nonempty) for all rounds $t \in [T]$.
\end{proposition}
\begin{proof}
    First, we show that if $x_t \in (1 - \alpha) \Yc$ and $\delta \leq \frac{2 \alpha \epsilon}{\sqrt{d} L D}$, then $x_t \in \Yc_t^p$.
    Under these conditions, we can choose $y = x_t$ to get that
    \begin{align*} 
        & g(x_t) + \frac{1}{2} \sqrt{d} L \delta D + \tilde{\nabla} g(x_t)^\top (y - x_t) + \frac{L}{2} \| y - x_t \|^2\\
        & \leq g(x_t) + \frac{1}{2} \sqrt{d} L \delta D\\
        & \leq g(x_t) + \alpha \epsilon\\
        & \leq - \alpha \epsilon + \alpha \epsilon = 0,
    \end{align*}
    where the first inequality uses the choice $y = x_t$, the second inequality uses the choice of $\delta$.
    The last inequality uses the fact that if $x_t \in (1 - \alpha) \Yc$, then $g(x_t) \leq - \alpha \epsilon$, because then there exists $z \in \Yc$ such that $(1 - \alpha)z =  x_t$ and therefore
    \begin{equation}
        \label{eqn:g_bound}
        g(x_t) = g( \alpha \mathbf{0} + (1 - \alpha) z) \leq \alpha g(\mathbf{0}) + (1 - \alpha) g(z) \leq - \alpha \epsilon.
    \end{equation}
    Therefore, $x_t \in \Yc_t^p$ if $x_t \in (1 - \alpha) \Yc$ and $\delta \leq \frac{2 \alpha \epsilon}{\sqrt{d} L D}$.

    Using this, we show the statement of the lemma with induction over $t$.
    The base case holds because $x_1 = \mathbf{0} \in (1 - \alpha) \Yc$ so $x_1 \in \Yc_1^p$.
    Then, suppose that $x_t \in \Yc_t^p$.
    It follows that there exists $\mu \in [0,1]$ such that $x_t + \mu (\xtil_{t+1} - x_t) \in \Yc_t^p$ (e.g. one can choose $\mu = 0$).
    Therefore, the update of $\gamma_t$ in line \ref{lne:safe_scal} is well-defined and indeed $x_t + \gamma_t (\xtil_{t+1} - x_t) \in \Yc_t^p$.
    It follows that
    \begin{equation*}
        x_{t+1} = (1 - \alpha) (x_t + \gamma_t (\xtil_{t+1} - x_t)) \in (1 - \alpha) \Yc_t^p \subseteq (1 - \alpha)\Yc,
    \end{equation*}
    where the second inclusion is from from Proposition \ref{prop:opt_pess}.
    Therefore, $x_{t+1} \in \Yc_{t+1}^p$ and the induction is complete.
\end{proof}

\subsection{Proof of Proposition \ref{prop:safe}}

\label{apx:safe}

\begin{proposition}[Duplicate of Proposition \ref{prop:safe}]
    \label{prop:safe2}
    Let Assumption \ref{ass:const} hold and assume that $x_t \in \Yc_t^p$ for all $t \in [T]$. If $\delta \leq \alpha r$, then all actions played by the algorithm, i.e. $x_t, x_t + \delta e_1, ..., x_t + \delta e_d$ for all $t$, are in the feasible set $\Yc$.
\end{proposition}
\begin{proof}
    Note that $x_t, x_t + \delta e_1, ..., x_t + \delta e_d$ are in $x_t +  \delta \Bb$ so it is sufficient to show that $x_t + \delta \Bb \subseteq \Yc$ for all $t$.
    Since $x_t \in \Yc_t^p$, the update of $\gamma_t$ in line \ref{lne:safe_scal} is well-defined and indeed $x_t + \gamma_t (\xtil_{t+1} - x_t) \in \Yc_t^p$.
    It follows that
    \begin{align*}
        x_{t+1} + \delta \Bb & = (1 - \alpha) (x_t + \gamma_t (\xtil_{t+1} - x_t)) + \alpha r \Bb \tag{a} \label{eqn:safe_a}\\
        & \subseteq (1 - \alpha) \Yc_t^p + \alpha r \Bb \tag{b} \label{eqn:safe_b}\\
        & \subseteq (1 - \alpha) \Yc + \alpha r \Bb \tag{c} \label{eqn:safe_c}\\
        & \subseteq (1 - \alpha) \Yc \oplus \alpha \Yc \tag{d} \label{eqn:safe_d}\\
        & = \Yc, \tag{e} \label{eqn:safe_e}
    \end{align*}
    where the \eqref{eqn:safe_a} is the update of $x_{t+1}$ in line \ref{lne:act_upd}, \eqref{eqn:safe_b} uses that $x_t + \gamma_t (\xtil_{t+1} - x_t) \in \Yc_t^p$, \eqref{eqn:safe_c} uses $\Yc_t^p \subseteq \Yc$ from Proposition \ref{prop:opt_pess}, \eqref{eqn:safe_d} uses $r\Bb \subseteq \Yc$ from Assumption \ref{ass:slaters}, and \eqref{eqn:safe_e} uses the convexity of $\Yc$.
\end{proof}

\section{Proof of Theorem \ref{thm:main}}
\label{apx:reg}

In this section, we first state the supporting lemmas for Theorem \ref{thm:main} in Appendix \ref{apx:tech_lem} and then give the proof of the theorem in Appendix \ref{apx:thm}.

\subsection{Supporting Lemmas}
\label{apx:tech_lem}

In this section, we give the proof of the Lemmas \ref{lem:gam}, \ref{lem:dist} and \ref{lem:lin_reg}.
First, we restate Lemma \ref{lem:gam} and then give the proof.

\begin{lemma}[Duplicate of Lemma \ref{lem:gam}]
    \label{lem:gam2}
    Let Assumptions \ref{ass:const} and \ref{ass:slaters} hold.
    If $\delta \leq \frac{2 (\kappa - 1) \alpha \epsilon}{(\kappa + 1) \sqrt{d} L D}$, then $\gamma_t \geq 1/\kappa$ for all $t \in [T]$.
\end{lemma}
\begin{proof}
    First, note that the requirements on $\delta$ ensure that the algorithm is well-defined in the sense of Proposition \ref{prop:viab} as,
    \begin{equation}
        \label{eqn:sat_delta}
        \delta \leq \frac{2 (\kappa - 1) \alpha \epsilon}{(\kappa + 1) \sqrt{d} L D} \leq \frac{2 \kappa \alpha \epsilon}{(\kappa + 1) \sqrt{d} L D} \leq \frac{2 \alpha \epsilon}{\sqrt{d} L D}.
    \end{equation}
    Then, let $y := \xtil_{t+1} - x_t$.
    Since $\xtil_{t+1}$ is in $\Yc_t^o$, we know that
    \begin{equation}
        \label{eqn:opt_rearr}
        \begin{split}
        & g(x_t)  - \frac{1}{2} \sqrt{d} L \delta D + \tilde{\nabla} g(x_t)^\top y + \frac{M}{2} \| y \|^2 \leq 0\\
        & \Longleftrightarrow \tilde{\nabla} g(x_t)^\top y + \frac{M}{2} \| y  \|^2 \leq -g(x_t) + \frac{1}{2} \sqrt{d} L \delta D. 
        \end{split}
    \end{equation}
    Then, we aim to find a $\mu \in [0,1]$ such that $u = x_t + \mu (\xtil_{t+1} - x_t) = x_t + \mu y$ is in $\Yc_t^p$.
    Since $\Xc$ is convex and $x_t$ and $\xtil_{t+1}$ are in $\Xc$, we know that $u$ is in $\Xc$ for any such $\mu$.
    Then, choosing $\mu = 1/\kappa$, we have that 
    \begin{align*} 
        & g(x_t) + \frac{1}{2} \sqrt{d} L \delta D + \tilde{\nabla} g(x_t)^\top (u - x_t) + \frac{L}{2} \| u - x_t \|^2\\
        &= g(x_t) + \frac{1}{2} \sqrt{d} L \delta D + \mu \tilde{\nabla} g(x_t)^\top y + \mu^2 \frac{L}{2} \| y \|^2\\
        &= g(x_t) + \frac{1}{2} \sqrt{d} L \delta D + \mu \left( \tilde{\nabla} g(x_t)^\top y + \mu \frac{L}{2} \| y \|^2 \right)\\
        & = g(x_t) +  \frac{1}{2} \sqrt{d} L \delta D + \mu \left( \tilde{\nabla} g(x_t)^\top y + \frac{M}{2} \| y \|^2 \right) \tag{a} \label{eqn:scala}\\
        & \leq g(x_t) +  \frac{1}{2} \sqrt{d} L \delta D + \mu \left( -g(x_t) + \frac{1}{2} \sqrt{d} L \delta D \right) \tag{b} \label{eqn:scalb}\\
        & \leq (1 - \mu) g(x_t) +  (1 + \mu) \frac{1}{2} \sqrt{d} L \delta D \\
        & \leq (1 - \mu) g(x_t) + (\mu - 1)g(x_t) = 0 \tag{c} \label{eqn:scalc},
    \end{align*}
    where \eqref{eqn:scala} uses the choice of $\mu$ and \eqref{eqn:scalb} uses \eqref{eqn:opt_rearr}.
    Line \eqref{eqn:scalc} uses the the fact that
    \begin{equation*}
        \delta \leq \frac{2 (\kappa - 1) \alpha \epsilon}{(\kappa + 1)\sqrt{d} L D} \leq \frac{-2 (\kappa - 1) g(x_t)}{(\kappa + 1)\sqrt{d} L D} = \frac{2 (\mu - 1)g(x_t)}{(1 + \mu)\sqrt{d} L D}
    \end{equation*}
    where we use the fact that $g(x_t) \leq -\alpha \epsilon$ since $x_t \in (1 - \alpha) \Yc$ as in \eqref{eqn:g_bound}.
    Finally, since $u = x_t + \mu (\xtil_{t+1} - x_t)$ is in $\Yc_t^p$ with $\mu = 1/\kappa$ and $\gamma_t$ is defined as the largest such $\mu$, we know that $\gamma_t \geq 1/\kappa$ by definition.
\end{proof}

Next, we restate Lemma \ref{lem:dist} and then give the proof in the following.

\begin{lemma}[Duplicate of Lemma \ref{lem:dist}]
    \label{lem:dist2}
    Let Assumptions \ref{ass:cost_funcs}, \ref{ass:const} and \ref{ass:slaters} hold.
    Fix any $\rho > 0$.
    If $\eta \leq \frac{1/\kappa}{2 d G(1-1/\kappa)} \rho$, $\alpha \leq \frac{1}{2 D \kappa} \rho$ and $\delta \leq \frac{2 (\kappa - 1) \alpha \epsilon}{(\kappa + 1) \sqrt{d} L D}$, then it holds that $\| x_t - \xtil_t \| \leq \rho$ for all~$t$.
\end{lemma}
\begin{proof}
    We show this by induction.
    The base case holds by definition as $\xtil_1 = x_1 = \mathbf{0}$.
    Suppose that $\| x_t - \xtil_t \| \leq \rho$, then we have that
    \begin{align*}
        \| \xtil_{t+1} - x_{t+1} \| & = \| \xtil_{t+1} - (1 - \alpha)\xtil_{t+1} + (1- \alpha)\xtil_{t+1} - x_{t+1} \|\\
        & \leq \| \xtil_{t+1} - (1 - \alpha)\xtil_{t+1}\| + \| (1- \alpha)\xtil_{t+1} - x_{t+1} \|\\
        & \leq \| (1- \alpha)\xtil_{t+1} - x_{t+1} \| + \alpha D\\
        & = \| (1- \alpha) \xtil_{t+1} - (1- \alpha)(x_t + \gamma_t (\xtil_{t+1} - x_t) )\| + \alpha D \tag{a} \label{eqn:taga}\\
        & = (1- \alpha) (1 - \gamma_t) \| \xtil_{t+1} - x_t \| + \alpha D\\
        & \leq (1- \alpha) (1 - 1/\kappa) \| \xtil_{t+1} - x_t \| + \alpha D \tag{b} \label{eqn:tagb}\\
        & = (1- \alpha) (1 - 1/\kappa) \| \Pi_{\Yc_t^o}(\xtil_t - \eta \tilde{\nabla} f_t (x_t) ) - x_t \| + \alpha D\\
        & \leq (1- \alpha) (1 - 1/\kappa) \| \xtil_t - \eta \tilde{\nabla} f_t (x_t) - x_t \| + \alpha D \tag{c} \label{eqn:tagc}\\
        & \leq (1- \alpha) (1 - 1/\kappa) (\| \xtil_t - x_t \| + \eta \| \tilde{\nabla} f_t (x_t) \|) + \alpha D \tag{d} \label{eqn:tagd}\\
        & \leq (1 - 1/\kappa) (\| \xtil_t - x_t \| + \eta \| \tilde{\nabla} f_t (x_t) \|) + \alpha D \\
        & \leq (1 - 1/\kappa) (\rho + \eta d G) + \alpha D \tag{e} \label{eqn:tage}\\
        & \leq (1 - 1/\kappa) \left( \rho + \frac{1/\kappa}{2 d G(1-1/\kappa)} \rho d G \right) + \frac{1/\kappa}{2} \rho\\
        & = (1 - 1/\kappa) \rho + \frac{1/\kappa}{2} \rho + \frac{1/\kappa}{2} \rho = \rho
    \end{align*}
    where \eqref{eqn:taga} uses the update for $x_{t+1}$ in line \ref{lne:act_upd} of the algorithm, \eqref{eqn:tagb} follows from Lemma \ref{lem:gam}, \eqref{eqn:tagc} follows from the fact that $x_t \in \Yc \subseteq \Yc_t^o$ from Propositions \ref{prop:opt_pess} and \ref{prop:safe}, \eqref{eqn:tagd} is the triangle inequality and \eqref{eqn:tage} uses the induction hypothesis and that $\| \tilde{\nabla} f_t (x_t) \| \leq d G$ from Proposition \ref{prop:grad_est}.
\end{proof}

Lastly, we restate Lemma \ref{lem:lin_reg} and give the proof.

\begin{lemma}[Duplicate of Lemma \ref{lem:lin_reg}]
    \label{lem:lin_reg2}
    Let Assumptions \ref{ass:cost_funcs} and \ref{ass:const} hold.
    Then, for any $v \in \Yc$, it holds that
    \begin{equation*}
        \tilde{\nabla} f_t (x_t)^\top (\xtil_t - v) \leq \frac
        {1}{2 \eta} \left( \| \xtil_t - v \|^2 - \| \xtil_{t+1} - v \|^2 \right) + \frac{1}{2} \eta d^2 G^2,
    \end{equation*}
    for all $t \in [T]$.
\end{lemma}
\begin{proof}
    Because $v \in \Yc \subseteq \Yc_t^o$, we know that
    \begin{align*}
        & \| \xtil_{t+1} - v \|^2\\
        & = \| \Pi_{\Yc_t^o}(\xtil_t - \eta \tilde{\nabla} f_t (x_t) ) - v \|^2\\
        & \leq \| \xtil_t - \eta \tilde{\nabla} f_t (x_t) - v \|^2\\
        & = \| \xtil_t - v \|^2 - 2 \eta \tilde{\nabla} f_t (x_t)^\top(\xtil_t - v) + \eta^2 \| \tilde{\nabla} f_t (x_t) \|^2\\
        & \leq \| \xtil_t - v \|^2 - 2 \eta \tilde{\nabla} f_t (x_t)^\top(\xtil_t - v) + \eta^2 d^2 G^2.
    \end{align*}
    The proof is complete by rearranging and dividing by $2 \eta$.
\end{proof}

\subsection{Proof of Theorem \ref{thm:main}}
\label{apx:thm}

\begin{theorem}[Duplicate of Theorem \ref{thm:main}]
    \label{thm:main2}
    Let Assumptions \ref{ass:cost_funcs}, \ref{ass:set_bound}, \ref{ass:const} and \ref{ass:slaters} hold.
    If $\alpha = \min(d G (1 - 1/\kappa) \eta /D, 1/2)$ and 
    \begin{equation*}
        \delta = \min\left(\frac{1}{\left( \frac{1}{2} \sqrt{d} L D + G \right) T}, \frac{2 (\kappa - 1) \alpha \epsilon}{(\kappa + 1) \sqrt{d} L D}, \alpha r \right),
    \end{equation*}
    then all actions chosen by MP-ROGD (Algorithm \ref{alg:main_alg}) are within the constraints, and the regret  satisfies
    \begin{equation*}
        R_T \leq 2 d^2 G^2 \left( \kappa - \frac{3}{4} \right) \eta T  + \frac{ D^2}{2 \eta} + 1.
    \end{equation*}
    Furthermore, choosing $\eta = \frac{D}{2 \sqrt{\left( d/4 +  \kappa - 1 \right) d G^2 T}}$ ensures that 
    \begin{equation*}
        R_T \leq 2 D G \sqrt{d \left( \frac{1}{4} d +  \kappa - 1 \right) T} + 1.
    \end{equation*}
\end{theorem}
\begin{proof}
    Since $x_*$ is in $\Yc$, we can use Lemma \ref{lem:lin_reg} with $v \leftarrow x_*$ and sum over $t$ to get
    \begin{equation}
        \label{eqn:opt_act}
        \begin{split}
        & \sum_{t=1}^T \tilde{\nabla} f_t (x_t)^\top (\xtil_t - x_*)\\
        & \leq  \frac{1}{2 \eta} \sum_{t=1}^T (\| \xtil_t - x_* \|^2 - \| \xtil_{t+1} - x_* \|^2) + \frac{1}{2} d^2 G^2 \eta T \\
        & = \frac{1}{2 \eta} (\| \xtil_1 - x_* \|^2 - \| \xtil_{T+1} - x_* \|^2) + \frac{1}{2} d^2 G^2 \eta T \\
        & \leq \frac{1}{2 \eta} D^2 + \frac{1}{2} d^2 G^2 \eta T.
        \end{split}
    \end{equation}
    Also, note that
    \begin{equation}
        \label{eqn:prev1}
        \begin{split}
            \frac{1}{k} \sum_{i=1}^{k} f_t(x_{t,i}) & = \frac{1}{d+1} \left( f_t(x_t) + \sum_{i=1}^{d} f_t(x_t + \delta e_i) \right)\\
            & \leq \frac{1}{d+1} \left( f_t(x_t) + \sum_{i=1}^{d} \left( f_t(x_t) + |f_t(x_t + \delta e_i) - f_t(x_t) | \right) \right)\\
            & \leq \frac{1}{d+1} \left( f_t(x_t) + \sum_{i=1}^{d} \left( f_t(x_t) + G \delta \right) \right)\\
            & = f_t(x_t) + \frac{d}{d + 1} G \delta \leq f_t(x_t) + G \delta.
        \end{split}
    \end{equation}
    Then, we can bound the regret directly as
    \begin{align*}
        R_T & = \sum_{t=1}^T \left( \frac{1}{k} \sum_{i=1}^{k} f_t(x_{t,i}) - f_t (x_*) \right) \\
        & \leq \sum_{t=1}^T \left( f_t(x_t) - f_t (x_*) \right) + \delta G T \tag{a} \label{eqn:prev}\\
        & \leq \sum_{t=1}^T \nabla f_t(x_t) ^\top (x_t - x_*) + \delta G T \tag{b} \label{eqn:conv} \\
        & = \sum_{t=1}^T \nabla f_t(x_t) ^\top (\xtil_t - x_*) + \sum_{t=1}^T \nabla f_t(x_t) ^\top (x_t - \xtil_t) + \delta G T\\
        & \leq \sum_{t=1}^T \nabla f_t(x_t) ^\top (\xtil_t - x_*) + 2 (\kappa - 1) d G^2 \eta T + \delta G T \tag{c} \label{eqn:cauchs}\\
        & \leq \sum_{t=1}^T \tilde{\nabla} f_t(x_t) ^\top (\xtil_t - x_*) + \sum_{t=1}^T (\nabla f_t (x_t) - \tilde{\nabla} f_t(x_t)) ^\top (\xtil_t - x_*) + 2 (\kappa - 1) d G^2 \eta T + \delta G T \\
        & \leq \sum_{t=1}^T \tilde{\nabla} f_t(x_t) ^\top (\xtil_t - x_*) + 2 (\kappa - 1) d G^2 \eta T + \left( \frac{1}{2} \sqrt{d} L D + G \right) T \delta \tag{d} \label{eqn:grad_error}\\
        & \leq \frac{ D^2}{2 \eta} + \frac{1}{2} d^2 G^2 \eta T + 2 (\kappa - 1) d G^2 \eta T  + \left( \frac{1}{2} \sqrt{d} L D + G \right) T \delta \tag{e} \label{eqn:all_tog}\\
        & = \frac{ D^2}{2 \eta} + 2 \left(\frac{1}{4} d + \kappa - 1 \right) d G^2 \eta T  + \left( \frac{1}{2} \sqrt{d} L D + G \right) T \delta \\
        & \leq 2 D G \sqrt{d \left( \frac{1}{4} d +  \kappa - 1 \right) T} + 1 \tag{f} \label{eqn:step_sze},
    \end{align*}
    where \eqref{eqn:prev} is due to \eqref{eqn:prev1}, \eqref{eqn:conv} is due to the convexity of $f_t$, \eqref{eqn:cauchs} is from applying Lemma \ref{lem:dist} with $\rho = 2(\kappa - 1) d G \eta$ and note that the conditions on $\alpha$ are $\delta$ are satisfied by specification, \eqref{eqn:grad_error} follows from applying Proposition \ref{prop:grad_est}, \eqref{eqn:all_tog} follows from Lemma \ref{lem:lin_reg}, and \eqref{eqn:step_sze} uses the choice of step size $\eta = \frac{D}{2 \sqrt{\left( d/4 +  \kappa - 1 \right) d G^2 T}}$ and gradient estimator radius $\delta \leq \frac{1}{\left( \frac{1}{2} \sqrt{d} L D + G \right) T}$.
\end{proof}

\end{document}